\documentclass[letterpaper]{article} 
\usepackage{aaai25}  
\usepackage{times}  
\usepackage{helvet}  
\usepackage{courier}  
\usepackage[hyphens]{url}  
\usepackage{graphicx} 
\usepackage{natbib}  
\usepackage{caption} 
\usepackage{algorithm}
\usepackage{algorithmic}
\usepackage{amsmath}
\usepackage{amssymb}
\usepackage{amsthm}
\usepackage{booktabs}
\usepackage{tikz}
\usetikzlibrary{positioning, arrows.meta, shapes.geometric, patterns, calc, fit, backgrounds, decorations.pathreplacing}
\usepackage{multirow}

\newtheorem{definition}{Definition}
\newtheorem{proposition}{Proposition}
\newtheorem{theorem}{Theorem}
\newtheorem{principle}{Principle}
\DeclareMathOperator*{\argmin}{argmin}

\title{Do Fair Models Reason Fairly? Counterfactual Explanation Consistency for Procedural Fairness in Credit Decisions}
\author{
    Gideon Popoola$^{1}$ and John Sheppard$^{1}$
}
\affiliations{
    \textsuperscript{\rm 1}Gianforte School of Computing

    Montana State University\\
    Bozeman, MT, USA, 59715\\
    gideon.popoola@student.montana.edu, john.sheppard@montana.edu
}

\begin{document}

\maketitle

\begin{abstract}
Machine learning algorithms in socially sensitive domains (e.g., credit decisions) often focus on equalizing predictive outcomes. However, satisfying these metrics does not guarantee that models use the same reasoning for different groups. We show that existing outcome-fair models can still apply fundamentally different reasoning to individuals, a ``hidden procedural bias'' missed by standard fairness metrics and algorithms. We propose Counterfactual Explanation Consistency (CEC), a framework that detects and mitigates this bias by aligning feature attributions between individuals and their counterfactual counterparts. Key contributions include a nearest-neighbor counterfactual generation method, a modified baseline for integrated gradient comparisons, an individual-level procedural fairness metric, and a corresponding training loss. We introduce a taxonomy identifying ``Regime B'' (same outcome, different reasoning) as a critical blind spot. Experiments on synthetic data, German Credit, Adult Income, and HMDA mortgage data demonstrate that outcome-fair baselines exhibit substantial hidden bias, while CEC substantially reduces it with modest utility cost.

\end{abstract}

\section{Introduction}
Machine learning (ML) models are now widely deployed in socially sensitive domains such as financial services for credit scoring, loan underwriting, and risk assessment \cite{bello2023machine, dastile2020statistical}. Regulatory frameworks such as the Equal Credit Opportunity Act (ECOA) and the Fair Housing Act prohibit lenders from disadvantaging individuals based on protected attributes, including race, gender, and ethnicity \cite{act2018equal, act1968fair}. Ensuring that ML models comply with these requirements has become a central challenge for practitioners and regulators alike.

Most algorithmic fairness research addresses this challenge through \emph{outcome fairness}, which focuses on comparing predictive outcomes across groups using criteria such as demographic parity or equalized odds \cite{hardt2016equality, feldman2015certifying}. While effective at reducing outcome disparities, these metrics leave a critical question unanswered, which is \textit{does the model rely on the same reasoning when evaluating individuals from different groups?}

This question has deep roots in theories of procedural justice. In the philosophical and legal traditions, fairness requires not only equitable outcomes but also consistent application of decision criteria \cite{rawls1971theory, leventhal1980should, thibaut1973procedural}. The principle that ``like cases should be treated alike'' demands that individuals with comparable qualifications be evaluated using the same standards, regardless of demographic membership. In lending regulation, ECOA explicitly prohibits \emph{disparate treatment} even when outcomes appear equal \cite{act2018equal}. A lender who evaluates Male applicants primarily on credit score but evaluates Female applicants primarily on employment history has engaged in prohibited conduct, regardless of whether approval rates are balanced. This legal and ethical framework motivates our focus on procedural fairness in automated decision systems.

Figure \ref{fig:motivation} illustrates the core problem. Two loan applicants with nearly identical financial profiles but different protected groups both receive approval, satisfying outcome fairness. However, examining the model's procedure reveals different reasoning. Applicant A's decision is driven by credit score and income, while Applicant B's decision depends primarily on employment history and collateral. This \emph{hidden procedural bias} (different reasoning pathways producing the same outcome) is invisible to most traditional fairness metrics yet constitutes disparate treatment under ECOA.

\begin{figure}[t]
\centering
\begin{tikzpicture}[
    font=\footnotesize,
    barA/.style={fill=blue!60, draw=blue!80},
    barB/.style={fill=red!50, draw=red!70},
    >=Stealth
]

\node[anchor=north west, font=\footnotesize\bfseries] at (0, 3.6) {Applicant A (Group 0)};
\node[anchor=north west, font=\scriptsize] at (0, 3.2) {Income=\$60K, Score=720, Approved};

\foreach \feat/\val/\ypos in {Credit Score/0.55/2.5, Income/0.28/2.0, Empl. History/0.10/1.5, Collateral/0.05/1.0, Other/0.02/0.5} {
    \draw[barA] (0, \ypos) rectangle (\val*5.5, \ypos+0.35);
    \node[anchor=east, font=\scriptsize] at (-0.05, \ypos+0.175) {\feat};
    \node[anchor=west, font=\scriptsize] at (\val*5.5+0.05, \ypos+0.175) {\val};
}

\node[anchor=north west, font=\footnotesize\bfseries] at (0, -0.2) {Applicant B (Group 1)};
\node[anchor=north west, font=\scriptsize] at (0, -0.6) {Income=\$62K, Score=725, Approved};

\foreach \feat/\val/\ypos in {Credit Score/0.08/-1.3, Income/0.02/-1.8, Empl. History/0.48/-2.3, Collateral/0.38/-2.8, Other/0.04/-3.3} {
    \draw[barB] (0, \ypos) rectangle (\val*5.5, \ypos+0.35);
    \node[anchor=east, font=\scriptsize] at (-0.05, \ypos+0.175) {\feat};
    \node[anchor=west, font=\scriptsize] at (\val*5.5+0.05, \ypos+0.175) {\val};
}

\node[draw=green!60!black, fill=green!10, rounded corners=2pt, font=\scriptsize, text width=2.2cm, align=center] at (4.5, 3.55) {Outcome Fair \checkmark\\Same decision};

\node[draw=red!70, fill=red!10, rounded corners=2pt, font=\scriptsize, text width=2.2cm, align=center] at (4.5, -0.12) {Procedurally Unfair $\times$\\Different reasoning};

\end{tikzpicture}
\caption{Two financially similar applicants from different demographic groups receive the same loan approval (outcome-fair), but the model's procedure reveals entirely different reasoning (procedurally unfair). This \emph{Regime B} bias is invisible to standard fairness metrics.}
\label{fig:motivation}
\end{figure}
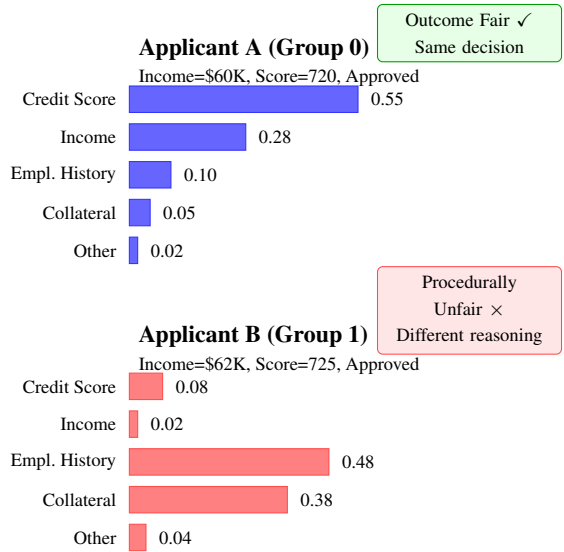

We formalize this phenomenon through a taxonomy of four fairness regimes based on whether a model's predictions and explanations remain consistent under counterfactual changes to the protected attribute (Table \ref{tab:taxonomy}). For an individual $(x, y, a)$ with counterfactual $\tilde{x}$ from the opposite group, we assess two properties: \emph{prediction consistency} (does the classification remain the same?) and \emph{explanation consistency} (does the model's reasoning remain the same?). The cross-product yields four regimes. Existing fairness metrics target Regimes C and D, where outcome disparities are observable. However, Regime B, in which predictions agree but reasoning differ is equally problematic from both legal and ethical standpoints, yet remains invisible to all standard metrics.

\begin{table}[t]
\centering
\small
\caption{Taxonomy of fairness regimes. CEC targets Regime~B, which is invisible to outcome-based methods.}
\label{tab:taxonomy}
\begin{tabular}{@{}lcc@{}}
\toprule
\textbf{Regime} & \textbf{Same Prediction?} & \textbf{Same Reasoning?} \\
\midrule
A (Fully Fair) & \checkmark & \checkmark \\
B (Hidden Bias) & \checkmark & $\times$ \\
C (Outcome Unfair) & $\times$ & \checkmark \\
D (Severely Unfair) & $\times$ & $\times$ \\
\bottomrule
\end{tabular}
\end{table}

To address this gap, we propose \emph{Counterfactual Explanation Consistency} (CEC), a framework for detecting and mitigating hidden procedural bias. The core idea is to compare integrated gradient attribution vectors between a factual individual and a counterfactual counterpart. If the model is procedurally fair, these attribution vectors should be similar. The model should weigh credit score, income, and other financial features in the same proportions regardless of which group the applicant belongs to. When they diverge, CEC quantifies the degree of hidden bias and provides a differentiable training signal to mitigate it. In this paper, we make the following contributions:
\begin{itemize}
\item We formalize hidden procedural bias and introduce a 2×2 taxonomy distinguishing outcome vs. procedural fairness violations.
\item We propose a nearest-neighbor counterfactual generation method that produces realistic matches by controlling for financial merit and creditworthiness outcome, without requiring causal graphs or structural equations.
\item  We propose \emph{Counterfactual Explanation Consistency (CEC)}, a metric that measures explanation stability under demographic counterfactuals, along with a differentiable training loss that jointly optimizes accuracy, outcome fairness, and explanation consistency.
\item We conduct a comprehensive evaluation across synthetic, benchmark, and real-world lending datasets demonstrating that outcome-fair baselines contain substantial hidden bias, which CEC substantially reduces with minimal utility cost.
\end{itemize}

\section{Related Work}

\subsection{Outcome Fairness in Machine Learning}

The algorithmic fairness literature has developed numerous statistical criteria for evaluating predictive equity. \emph{Demographic parity} requires equal positive prediction rates across groups \cite{feldman2015certifying}; \emph{equalized odds} requires equal true positive and false positive rates \cite{hardt2016equality}; and \emph{calibration} requires that predicted probabilities reflect true outcomes within each group \cite{chouldechova2017fair}. These criteria are known to be mutually incompatible in general \cite{kleinberg2016inherent}, leading to a rich literature on navigating trade-offs.

Mitigating bias spans the machine learning pipeline. Pre-processing methods transform training data to remove correlations with protected attributes \cite{feldman2015certifying, kamiran2012data, popoola2024investigating}. In-processing methods incorporate fairness constraints during training through reductions \cite{agarwal2018reductions}, adversarial objectives \cite{zhang2018mitigating}, or constrained optimization \cite{cotter2019optimization}. Post-processing methods adjust decision thresholds after training \cite{hardt2016equality}. While these approaches effectively reduce outcome disparities, they evaluate only \emph{what} a model predicts, not \emph{how} it arrives at that prediction. Our work shows that this gap allows models to harbor hidden procedural bias even when all outcome constraints are satisfied.

\subsection{Explainability and Fairness}
Post-hoc explanation methods such as  Local Interpretable Model-Agnostic Explanations (LIME) \cite{ribeiro2016should}, SHapley Additive exPlanations (SHAP) \cite{lundberg2017unified}, and integrated gradients (IG) \cite{sundararajan2017axiomatic} have become essential tools for understanding model behavior in high-stakes domains. A growing body of work connects explainability with fairness. \citeauthor{dai2022fairness} (\citeyear{dai2022fairness}) examine whether explanation quality differs across demographic groups, finding that models may provide less informative explanations for minority groups. \citeauthor{begley2020explainability} (\citeyear{begley2020explainability}) propose using feature importance to audit models for discriminatory patterns. \citeauthor{agarwal2022openxai} (\citeyear{agarwal2022openxai}) provide benchmarks for evaluating explanation methods, and \citeauthor{slack2020fooling} (\citeyear{slack2020fooling}) show that explanations can be manipulated to hide bias.

\subsection{Procedural Fairness}

The concept of procedural fairness originates in social psychology and legal theory. \citeauthor{leventhal1980should} (\citeyear{leventhal1980should}) identified \emph{consistency} as a core component of procedural justice in which the same decision rules should apply across persons and across time. \citeauthor{grgic2018beyond} (\citeyear{grgic2018beyond}) study which features humans consider fair to use in algorithmic decisions, finding strong preferences for process-based criteria. In the machine learning context, procedural fairness has received less attention than outcome fairness, though recent work has begun to formalize process-based notions \cite{zhao2023fairness, germino2025explanation}. \citeauthor{dwork2012fairness} (\citeyear{dwork2012fairness}) proposed \emph{individual fairness} (similar individuals should receive similar outcomes), which captures a related but distinct intuition. Our method extends this notion by requiring not just similar outcomes but similar \emph{reasoning processes}.

\subsection{Counterfactual Fairness and Reasoning}

Counterfactual reasoning provides a natural framework for individual-level fairness analysis. \citeauthor{kusner2017counterfactual} (\citeyear{kusner2017counterfactual}) formalize counterfactual fairness using structural causal models (SCMs), requiring that predictions remain unchanged under interventions on protected attributes. Extensions address causal pathway constraints \cite{wu2019counterfactual, chiappa2019path} and relaxations for approximate fairness. 
However, building such graphs accurately is difficult in practice.

Counterfactual \emph{explanation} methods such as DiCE \cite{mothilal2020explaining}, FACE \cite{poyiadzi2020face}, and CARLA \cite{pawelczyk2021carla} generate alternative inputs that would change a prediction, focusing on \emph{recourse} (how can an individual obtain a different outcome?) rather than fairness auditing. We use counterfactual reasoning to assess whether \emph{explanations} remain consistent across demographic groups, and we achieve this without requiring causal graphs through merit-based nearest-neighbor matching.

\subsection{Multi-Objective Fair Learning}
Fair classification can be regarded as multi-objective, involving potential trade-offs between accuracy and one or more fairness constraints \cite{wang2024generating, cotter2019optimization}. Existing methods typically balance two objectives, which are predictive performance and outcome fairness \cite{wei2022fairness, nagpal2025optimizing}. Our training objective extends this to a three-dimensional trade-off by introducing explanation consistency as an additional objective, demonstrating that procedural fairness can be achieved jointly with outcome fairness and accuracy at modest cost.

\section{Methodology}
\subsection{Notation and Problem Setup}
We consider binary classification in a high-stakes, finance-based decision domain such as credit lending. Let $\mathcal{X} \subseteq \mathbb{R}^d$ denote the feature space, $\mathcal{Y} = \{0,1\}$ the label space (e.g., $y=1$ for loan approval), and $a \in \mathcal{A} = \{0,1\}$ a binary protected attribute (e.g., race or gender). Given training data $\mathcal{D} = \{(x_i, y_i, a_i)\}_{i=1}^n$, we learn a scoring function $f_\theta: \mathcal{X} \to \mathbb{R}$ parameterized by $\theta$, with predicted label $\hat{y} = \mathbb{I}[f_\theta(x) \geq \tau]$ for some threshold $\tau$.

A key part of our framework is the distinction between financial and non-financial features.

\begin{definition}[Financial Feature Set]
\label{def:financial_features}
Let $\mathcal{F} \subseteq \{1, \ldots, d\}$ denote the indices of \emph{financial features}. A feature can be regarded as a financial feature iff it satisfies three criteria:
\begin{enumerate}
    \item  \textbf{Merit-based}: they reflect creditworthiness or repayment ability (e.g., income, credit score, debt-to-income ratio),
    \item \textbf{Legally permissible}: they are not prohibited by fair lending regulation, and
    \item  \textbf{Not demographic proxies}: they are not strong proxies for protected attributes (e.g., residential zip code is excluded due to correlation with race from historical redlining).
\end{enumerate}
\end{definition}
In practice, $\mathcal{F}$ is determined through domain expertise and regulatory guidance. For credit lending, typical members include income, credit score, length of credit history, debt-to-income ratio, employment length, and liquid assets. Features typically excluded include residential address, educational institution, and certain occupation categories that may be demographically imbalanced.

Our method, CEC, ensures that $f_\theta$ uses consistent reasoning across demographic groups for individuals who are comparable on $\mathcal{F}$. The framework consists of three components:
1) counterfactual generation, 2) consistent baseline selection, and 3) the CEC metric with its training loss.

\subsection{Counterfactual Generation}

\subsubsection{The Problem with Naive Counterfactuals.}
A naive approach to generating counterfactuals simply flips the protected attribute while holding all other features constant. Given individual $x$ with protected attribute $a$, the na\"{i}ve counterfactual is $\tilde{x}_{\text{na\"{i}ve}} = (x_1, \ldots, x_{j-1}, 1{-}a, x_{j+1}, \ldots, x_d)$ where $j$ indexes the protected attribute. This approach suffers from a critical flaw, which is that many features are correlated with protected attributes due to historical and structural factors. Residential zip code correlates with race due to segregation, certain employment sectors are gender-imbalanced, and educational backgrounds track socioeconomic status, both of which are regarded as non-demographic proxies. Flipping race while holding zip code constant produces an individual who may not exist statistically, meaning the combination is rarely or never observed in the real population.

\subsubsection{Label-Stratified Nearest-Neighbor Matching.}
We propose a novel approach that avoids causal modeling while producing realistic and interpretable counterfactuals. The key insight is that for fairness auditing in lending, we need only identify individuals from different demographic groups who are comparable in terms of financial features that \emph{should} legitimately determine creditworthiness.

We approximate counterfactuals using matched opposite-group instances with similar financial characteristics. This follows standard practice in counterfactual fairness auditing when structural causal models are unavailable \cite{schwab2018perfect}. We do not claim causal counterfactual validity. Rather, we use matched proxies to approximate demographic interventions while preserving plausibility in financial features.

\begin{definition}[Counterfactual Matching]
\label{def:counterfactual}
Given individual $(x, y, a)$, financial feature set $\mathcal{F}$, and training data $\mathcal{D}$, a counterfactual is:
\[
    \tilde{x} = x_{i^*}, \;\; i^* = \argmin_{i:\, a_i = 1{-}a,\, y_i = y} \|x_\mathcal{F}^{\text{std}} - (x_i)_\mathcal{F}^{\text{std}}\|_2
\]
where $x_\mathcal{F}^{\text{std}}$ denotes the standardized subvector of $x$ restricted to $\mathcal{F}$.
\end{definition}

Stratifying on the true label $y$ is an important design choice because it ensures that the factual and counterfactual have identical ground-truth creditworthiness, so any measured explanation difference reflects group-dependent reasoning rather than legitimate risk differences. We note that observed labels may not perfectly capture ground-truth creditworthiness due to selective labeling (only approved applicants have observed repayment outcomes). This is a common limitation shared by all financial modeling, especially in credit decision-making. The counterfactual $\tilde{x}$ is always a \emph{real individual} from the training data, not a synthetic example, ensuring distributional plausibility. Algorithm \ref{alg:cf_generation} and Figure \ref{fig:phase1} show this procedure with quality safeguards.

\begin{algorithm}[t]
\caption{Label-Stratified Counterfactual Generation}
\label{alg:cf_generation}
\small
\begin{algorithmic}[1]
\REQUIRE Training data $\mathcal{D} = \{(x_i, y_i, a_i)\}_{i=1}^n$, financial features $\mathcal{F}$, distance threshold $\tau$ 
\ENSURE Counterfactual map $\mathcal{C}: i \mapsto \tilde{x}_i$

\STATE \textbf{// Preprocessing (one-time)}
\STATE Compute standardization: $\mu_\mathcal{F}, \sigma_\mathcal{F}$ from training data
\FOR{$(y, a) \in \{0,1\} \times \{0,1\}$}
    \STATE $\mathcal{D}_{y,a} \gets \{(x_i, i) : y_i = y, \, a_i = a\}$
    \IF{$|\mathcal{D}_{y,a}| > 0$}
        \STATE Standardize: $X_{y,a}^\mathcal{F} \gets \{(x_i)_\mathcal{F}^{\text{std}}\}$
        \STATE Build KD-tree index $\mathcal{I}_{y,a}$ on $X_{y,a}^\mathcal{F}$
    \ENDIF
\ENDFOR
\STATE \textbf{// Query (for each instance)}
\FOR{$i = 1, \ldots, n$}
    \STATE $\bar{a} \gets 1 - a_i$ \COMMENT{Opposite group}
    \IF{$|\mathcal{D}_{y_i, \bar{a}}| = 0$}
        \STATE Mark $i$ as unmatched; \textbf{continue}
    \ENDIF
    \STATE $(j^*, d^*) \gets \mathcal{I}_{y_i, \bar{a}}.\text{query}((x_i)_\mathcal{F}^{\text{std}})$
    \IF{$\tau > 0$ \AND $d^* > \tau$}
        \STATE Mark $i$ as unmatched \COMMENT{Poor quality}
    \ELSE
        \STATE $\mathcal{C}[i] \gets x_{j^*}$ \COMMENT{Full feature vector}
    \ENDIF
\ENDFOR
\end{algorithmic}
\end{algorithm}

The algorithm operates in two phases. The \textbf{Preprocessing} phase partitions the training data into four subsets by $(y, a)$, standardizes financial features using training-set statistics to ensure scale-invariant distance computation, and builds KD-tree indices for each non-empty partition. The complexity is $O(n \log n)$ per subset. \textbf{Querying} finds the nearest neighbor in the opposite-group, same-label partition for each instance, with $O(|\mathcal{F}| \log n)$ cost per query. The additional distance threshold $\tau$ rejects poor-quality matches when no suitable counterpart exists. 

\begin{figure}[t]
\centering
\begin{tikzpicture}[
    font=\small,
    box/.style={draw, rounded corners=3pt, minimum height=0.85cm, text width=2.6cm, align=center, font=\footnotesize},
    smallbox/.style={draw, rounded corners=2pt, minimum height=0.65cm, text width=2.1cm, align=center, font=\scriptsize},
    arrow/.style={-Stealth, thick},
    >=Stealth
]
 
\node[font=\footnotesize\bfseries, text=blue!70!black] at (0.0, 3.5) {Labeled Stratified NN Matching};
 
\node[box, fill=blue!8] (data) at (0, 2.5) {Training Data\\$\{(x_i, y_i, a_i)\}$};
 
\node[box, fill=blue!12] (partition) at (3.6, 2.5) {Partition by\\$(y, a) \in \{0,1\}$};
 
\node[smallbox, fill=green!10] (kdtree) at (1.8, 1.0) {Build KD-tree\\$\mathcal{I}_{y,a}$ on $\mathcal{F}$};
\node[smallbox, fill=orange!10] (baseline) at (5.2, 1.0) {Label-Group\\Baselines $b_{y,a}$};
 
\node[smallbox, fill=yellow!15] (match) at (1.8, -0.4) {NN Matching\\$\tilde{x}_i \gets$ query $\mathcal{I}_{y_i, 1-a_i}$};
 
\node[box, fill=yellow!25, text width=3.5cm] (output) at (3.5, -1.8) {Paired Dataset\\$\{(x_i, \tilde{x}_i, y_i, a_i, b_{y_i,a_i})\}_{i=1}^n$};
 
\draw[arrow] (data) -- (partition);
\draw[arrow] (partition) -- (kdtree) node[midway, left, font=\scriptsize] {subsets};
\draw[arrow] (partition) -- (baseline);
\draw[arrow] (kdtree) -- (match);
\draw[arrow] (match) -- (output);
\draw[arrow] (baseline) -- (output);
 
\node[font=\scriptsize, text=gray, text width=2.4cm, align=center] at (4.9, -0.3) {label\\match};
 
\end{tikzpicture}
\caption{Training data is partitioned by label and group. KD-tree indices enable efficient nearest-neighbor matching on financial features $\mathcal{F}$, producing paired data with label-group baselines for training.}
\label{fig:phase1}
\end{figure}
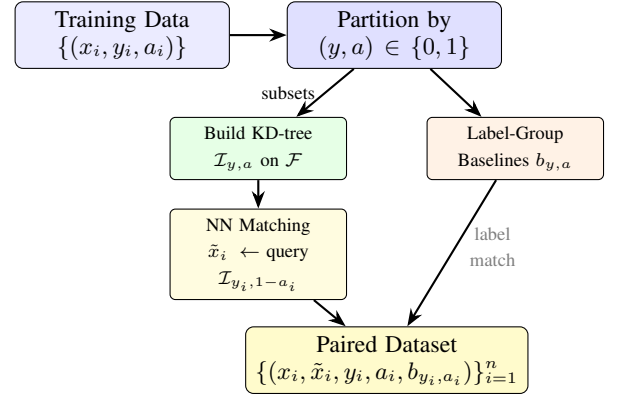

\subsection{The Consistent Baseline Principle}
 
\subsubsection{Integrated Gradients Background.}
 Integrated gradients (IG) \cite{sundararajan2017axiomatic} compute feature attributions by accumulating gradients along a straight-line path from a baseline $b$ to the input $x$:
\[
    \text{IG}_j(x; b) = (x_j - b_j) \int_0^1 \frac{\partial f}{\partial x_j}\bigg|_{b + \alpha(x-b)} d\alpha
\]
IG satisfies two desirable axioms: \emph{completeness} (attributions sum to $f(x) - f(b)$) and \emph{sensitivity} (if changing feature $j$ changes the prediction, it receives a nonzero attribution). The choice of baseline $b$ determines the reference point from which feature importance is measured.

\subsubsection{The Confound of Group-Specific Baselines.}
 
When comparing explanations across groups for a factual--counterfactual pair $(x, \tilde{x})$, natural choices are group-specific baselines $b_0$ for Group~0 and $b_1$ for Group~1. However, this confounds two distinct effects.
Consider an example where we suppose $x$ has income \$60K and belongs to Group~0 (average income $b_0^{\text{inc}} = \$50\text{K}$), while $\tilde{x}$ has income \$62K and belongs to Group~1 (average income $b_1^{\text{inc}} = \$55\text{K}$). Even if the model weighs income identically for both groups, the attributions differ: $\text{IG}_{\text{inc}}(x; b_0) \propto (60{-}50) = 10$ versus $\text{IG}_{\text{inc}}(\tilde{x}; b_1) \propto (62{-}55) = 7$. The difference of $3$ arises entirely from the different baselines, not from discriminatory reasoning, making. This makes the comparison invalid as a fairness measure.
 
\begin{principle}[Consistent Baseline]
\label{princ:baseline}
For a factual--counterfactual pair $(x, \tilde{x})$ where $x$ has label $y$ and attribute $a$, both attributions must use the same baseline:
\[
    \text{IG}(x;\, b_{y,a}) \quad \text{and} \quad \text{IG}(\tilde{x};\, b_{y,a})
\]
where $b_{y,a} = \frac{1}{|\{i: y_i=y,\, a_i=a\}|}\sum_{i: y_i=y,\, a_i=a} x_i$ is the label-group mean for the factual's label and group.
\end{principle}
 
Using the factual's group baseline, the example becomes: $\text{IG}_{\text{inc}}(x; b_0) \propto (60{-}50) = 10$ and $\text{IG}_{\text{inc}}(\tilde{x}; b_0) \propto (62{-}50) = 12$. The difference of $2$ now reflects the genuine input difference, with no baseline-induced confounder. Any disproportionate attribution change beyond what the input difference warrants reveals discriminatory model behavior. We formalize this guarantee as follows.

\begin{theorem}[Baseline Consistency Isolates Discrimination]
\label{thm:baseline}
Let $f$ be a differentiable function and $(x, \tilde{x})$ a pair with consistent baseline $b$. If the model's gradients on non-protected features are approximately constant along both integration paths (i.e., $\frac{\partial f}{\partial x_k}|_{b + \alpha(x-b)} \approx \frac{\partial f}{\partial x_k}|_x$ for $\alpha \in [0,1]$), then for non-protected feature $k$:
\[
    \text{IG}_k(x; b) - \text{IG}_k(\tilde{x}; b) \approx (x_k - \tilde{x}_k) \cdot \frac{\partial f}{\partial x_k}\bigg|_x
\]
\end{theorem}

\begin{proof}
By the definition of integrated gradients, for a non-protected feature $k$:
\[
\text{IG}_k(x; b) = (x_k - b_k)\int_0^1 \frac{\partial f}{\partial x_k}\bigg|_{b + \alpha(x - b)} d\alpha \
\text{IG}_k(\tilde{x}; b) 
\]
\[
= (\tilde{x}k - b_k)\int_0^1 \frac{\partial f}{\partial x_k}\bigg|_{b + \alpha(\tilde{x} - b)} d\alpha
\]
Under the assumption that gradients are approximately constant along both integration paths, i.e., $\frac{\partial f}{\partial x_k}\big|_{b + \alpha(x - b)} \approx \frac{\partial f}{\partial x_k}\big|_x$ and $\frac{\partial f}{\partial x_k}\big|_{b + \alpha(\tilde{x} - b)} \approx \frac{\partial f}{\partial x_k}\big|_x$ for all $\alpha \in [0,1]$, the integrals collapse:
\[
\text{IG}_k(x; b) \approx (x_k - b_k) \cdot \frac{\partial f}{\partial x_k}\bigg|_x \
\text{IG}_k(\tilde{x}; b) \approx (\tilde{x}_k - b_k) \cdot \frac{\partial f}{\partial x_k}\bigg|_x
\]
Subtracting yields:
\[
\text{IG}_k(x; b) - \text{IG}_k(\tilde{x}; b) \approx (x_k - \tilde{x}_k) \cdot \frac{\partial f}{\partial x_k}\bigg|_x
\]
Note that the baseline $b$ cancels entirely in the subtraction.
In contrast, with group-specific baselines $b^a$ and $b^{1-a}$, the same derivation yields:
\[
\text{IG}_k(x; b^a) - \text{IG}_k(\tilde{x}; b^{1-a}) \
\approx (x_k - b^a_k) \cdot \frac{\partial f}{\partial x_k}\bigg|_x - (\tilde{x}_k - b^{1-a}_k) \cdot 
\]
\[
\frac{\partial f} {\partial x_k}\bigg|_x \
= (x_k - \tilde{x}_k) \cdot \frac{\partial f}{\partial x_k}\bigg|_x + (b^{1-a}_k - b^a_k) \cdot \frac{\partial f}{\partial x_k}\bigg|_x.
\]
This introduces a spurious term $(b^{1-a}_k - b^a_k) \cdot \frac{\partial f}{\partial x_k}\big|_x$ that depends on group distribution differences rather than model discrimination.
\end{proof}

\noindent\textit{Remark.} Theorem~\ref{thm:baseline} is stated for the idealized case of approximately constant gradients; in practice, the nearest-neighbor counterfactuals from Definition~\ref{def:counterfactual} differ on multiple features, not just the protected attribute. The approximation tightens as match quality improves (smaller $\|x_\mathcal{F} - \tilde{x}_\mathcal{F}\|$), which we monitor via the coverage and distance metrics reported by Algorithm~\ref{alg:cf_generation}. The key guarantee is that attribution differences depend on input differences $(x_k - \tilde{x}_k)$ and model gradients, with no baseline-induced confounder. In contrast, group-specific baselines introduce an additional $(b_{y, a} - b_{y,1-a})$ term that conflates demographic distribution differences with discriminatory reasoning.

\subsection{The CEC Metric and Theoretical Properties}
\subsubsection{Normalization.}
IG vectors may have vastly different magnitudes across individuals. To compare explanation \emph{direction} rather than magnitude, we normalize:
\[
    \text{IG}_{\text{norm}}(x; b) = \frac{\text{IG}(x; b)}{\|\text{IG}(x; b)\|_2 + \epsilon}
\]
where $\epsilon = 10^{-8}$ prevents division by zero.
 
\subsubsection{The CEC Score.}
 
\begin{definition}[CEC Score]
\label{def:cec}
For a factual--counterfactual pair $(x, \tilde{x})$ with consistent baseline $b = b_{y,a}$, the raw L2 distance between normalized attribution vectors is:
\[
    \Delta_{\text{raw}}(x, \tilde{x}) = \left\| \text{IG}_{\text{norm}}(x; b) - \text{IG}_{\text{norm}}(\tilde{x}; b) \right\|_2
\]
Since both vectors have unit norm, $\Delta_{\text{raw}} \in [0, 2]$. We normalize to obtain a score in $[0, 1]$:
\[
    \Delta_{\text{CEC}}(x, \tilde{x}) = \!\left(\frac{\Delta_{\text{raw}}(x, \tilde{x})}{2}\right)
\]
where division by $2$ (the maximum L2 distance between unit vectors) maps the score to $[0, 1]$. The population-level CEC is:
\[
    \text{CEC}(f; \mathcal{D}) = \frac{1}{n}\sum_{i=1}^n \Delta_{\text{CEC}}(x_i, \tilde{x}_i)
\]
\end{definition}
 
\subsubsection{Theoretical Properties.}
The CEC metric admits a geometric interpretation, where each normalized attribution vector lies on the unit hypersphere in $\mathbb{R}^d$. The raw distance $\Delta_{\text{raw}}$ measures the chord length between two points on this hypersphere, related to angular separation by $\Delta_{\text{raw}} = \sqrt{2(1 - \cos\theta)}$. The normalization by $2$ converts this to a $[0,1]$ scale where the score equals $\frac{1}{2}\sqrt{2(1-\cos\theta)}$, providing an intuitive interpretation, where $0$ means identical explanations ($\theta = 0$) and $1$ means maximally opposed explanations ($\theta = \pi$).
 
\begin{proposition}[Bounded Range]
\label{prop:bounded}
$0 \leq \Delta_{\text{CEC}}(x, \tilde{x}) \leq 1$. The score is $0$ when explanations are identical (attribution vectors are parallel) and $1$ when they are maximally opposed (antipodal on the unit hypersphere).
\end{proposition}

\begin{proof}
Let $g = \text{IG}_{\text{norm}}(x; b)$ and $\tilde{g} = \text{IG}_{\text{norm}}(\tilde{x}; b)$. By construction, $\|g\|_2 = \|\tilde{g}\|_2 = 1$. Expanding the squared norm of their difference:
\[
\|g - \tilde{g}\|_2^2 = \|g\|_2^2 - 2\,g^\top\tilde{g} + \|\tilde{g}\|_2^2 = 2 - 2\cos\theta
\]
where $\theta \in [0, \pi]$ is the angle between $g$ and $\tilde{g}$. Since $\cos\theta \in [-1, 1]$, we have $\|g - \tilde{g}\|_2^2 \in [0, 4]$ and thus $\Delta_{\text{raw}} = \|g - \tilde{g}\|_2 \in [0, 2]$. Division by $2$ yields $\Delta_{\text{CEC}} = \Delta_{\text{raw}}/2 \in [0, 1]$.

The lower bound $\Delta_{\text{CEC}} = 0$ is attained when $\theta = 0$, i.e., $g = \tilde{g}$ (identical attribution directions). The upper bound $\Delta_{\text{CEC}} = 1$ is attained when $\theta = \pi$, i.e., $g = -\tilde{g}$ (antipodal vectors on the unit hypersphere, indicating maximally opposed explanations).
\end{proof}
 
\begin{proposition}[Sensitivity to Discrimination]
\label{prop:sensitivity}
If the model weighs feature $k$ differently across groups (i.e., $\frac{\partial f}{\partial x_k}\big|_x \neq \frac{\partial f}{\partial x_k}\big|_{\tilde{x}}$), then $\Delta_{\text{CEC}}(x, \tilde{x}) > 0$ for generic model parameterizations (exact cancellation by opposing changes in other features is non-generic).
\end{proposition}

\begin{proof}
Let $g = \text{IG}{\text{norm}}(x; b)$ and $\tilde{g} = \text{IG}{\text{norm}}(\tilde{x}; b)$ denote the normalized attribution vectors. Since $|g| = |\tilde{g}| = 1$, we have $\Delta_{\text{CEC}} = \frac{1}{2}|g - \tilde{g}|_2 = 0$ if and only if $g = \tilde{g}$, i.e., the normalized attribution vectors are identical.
Suppose $\frac{\partial f}{\partial x_k}\big|x \neq \frac{\partial f}{\partial x_k}\big|{\tilde{x}}$ for some feature $k$. By the IG formula, this gradient difference propagates into the unnormalized attributions: $\text{IG}_k(x; b)$ depends on $\frac{\partial f}{\partial x_k}$ evaluated along the path from $b$ to $x$, while $\text{IG}_k(\tilde{x}; b)$ depends on $\frac{\partial f}{\partial x_k}$ evaluated along the path from $b$ to $\tilde{x}$. 
Since the model weighs feature $k$ differently at $x$ versus $\tilde{x}$, the unnormalized vectors $\text{IG}(x; b)$ and $\text{IG}(\tilde{x}; b)$ differ in at least the $k$-th component. After $L_2$ normalization, $g = \tilde{g}$ would require all component-wise ratios to be equal: $\frac{\text{IG}_j(x;b)}{\text{IG}j(\tilde{x};b)} = c$ for all $j$ and some constant $c > 0$. 
This is a system of $d-1$ independent constraints on the model parameters. Since the model parameter space is continuous, the set of parameters that satisfy all $d-1$ constraints simultaneously has measure zero (it defines a manifold of codimension $d-1$). Thus, for generic model parameterizations, $g \neq \tilde{g}$ and $\Delta{\text{CEC}} > 0$. 
\end{proof}

These properties ensure that CEC reliably detects procedural discrimination whenever the model applies different feature weights to different groups. We note that for a model that treats all groups identically, residual $\Delta_{\text{CEC}} > 0$ may still arise from input differences between $x$ and $\tilde{x}$; however, this residual is bounded by match quality and does not reflect discriminatory reasoning.

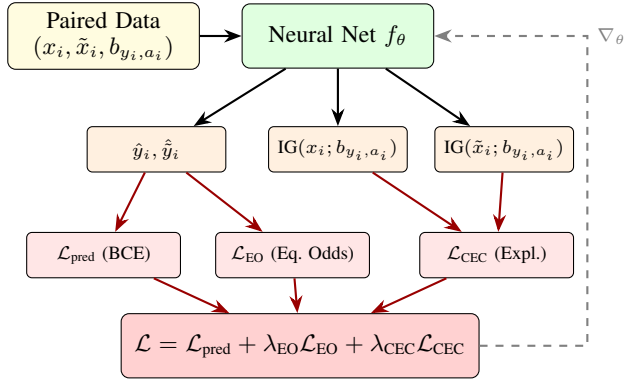
\begin{figure}[t]
\centering
\begin{tikzpicture}[
    font=\small,
    box/.style={draw, rounded corners=3pt, minimum height=0.85cm, text width=2.3cm, align=center, font=\footnotesize},
    smallbox/.style={draw, rounded corners=2pt, minimum height=0.6cm, text width=1.6cm, align=center, font=\scriptsize},
    arrow/.style={-Stealth, thick},
    lossarrow/.style={-Stealth, thick, red!60!black},
    >=Stealth
]
 
\node[font=\footnotesize\bfseries, text=red!60!black] at (0.0, 4.2) {Multi-Objective Training};
 
\node[box, fill=yellow!15] (pairs) at (0, 3.2) {Paired Data\\$(x_i, \tilde{x}_i, b_{y_i,a_i})$};
 
\node[box, fill=green!12] (model) at (3.1, 3.2) {Neural Net $f_\theta$};
 
\node[smallbox, fill=orange!12] (pred) at (0.7, 1.7) {$\hat{y}_i, \hat{\tilde{y}}_i$};
\node[smallbox, fill=orange!12] (ig1) at (3.1, 1.7) {IG$(x_i; b_{y_i,a_i})$};
\node[smallbox, fill=orange!12] (ig2) at (5.3, 1.7) {IG$(\tilde{x}_i; b_{y_i,a_i})$};
 
 
\node[smallbox, fill=red!10, text width=1.8cm] (lpred) at (0.0, 0.3) {$\mathcal{L}_{\text{pred}}$ (BCE)};
\node[smallbox, fill=red!10, text width=1.8cm] (leo) at (2.5, 0.3) {$\mathcal{L}_{\text{EO}}$ (Eq.\ Odds)};
\node[smallbox, fill=red!10, text width=1.8cm] (lcec) at (5.2, 0.3) {$\mathcal{L}_{\text{CEC}}$ (Expl.)};
 
\node[box, fill=red!18, text width=4.5cm] (total) at (2.6, -0.9) {$\mathcal{L} = \mathcal{L}_{\text{pred}} + \lambda_{\text{EO}}\mathcal{L}_{\text{EO}} + \lambda_{\text{CEC}}\mathcal{L}_{\text{CEC}}$};
 
\draw[arrow] (pairs) -- (model);
\draw[arrow] (model) -- (pred);
\draw[arrow] (model) -- (ig1);
\draw[arrow] (model) -- (ig2);
\draw[lossarrow] (pred) -- (lpred);
\draw[lossarrow] (pred) -- (leo);
\draw[lossarrow] (ig1) -- (lcec);
\draw[lossarrow] (ig2) -- (lcec);
\draw[lossarrow] (lpred) -- (total);
\draw[lossarrow] (leo) -- (total);
\draw[lossarrow] (lcec) -- (total);
 
\draw[arrow, dashed, gray] (total.east) -| (6.4, 3.2) -- (model.east) node[pos=0.0, right, font=\scriptsize, text=gray] {$\nabla_\theta$};
 
\end{tikzpicture}
\caption{Phase~2 (Training): Each minibatch passes through the model to compute predictions and integrated gradients for both factual and counterfactual inputs (using the same baseline). Three losses, prediction accuracy, equalized odds, and explanation consistency are combined and backpropagated.}
\label{fig:phase2}
\end{figure} 
 
\subsection{Training Objective}
We incorporate CEC as a differentiable regularizer alongside prediction and outcome fairness losses as shown in Figure \ref{fig:phase2} and Algorithm \ref{alg:training}:
\[
\label{eq:full_loss}
    \min_\theta \;\; \mathcal{L}_{\text{pred}}(\theta) + \lambda_{\text{EO}}\, \mathcal{L}_{\text{EO}}(\theta) + \lambda_{\text{CEC}}\, \mathcal{L}_{\text{CEC}}(\theta)
\]
where $\mathcal{L}_{\text{pred}}$ is the binary cross-entropy loss:
\[
    \mathcal{L}_{\text{pred}} = -\frac{1}{m}\sum_{i=1}^m \left[y_i \log \sigma(f_\theta(x_i)) + (1{-}y_i)\log(1{-}\sigma(f_\theta(x_i)))\right]
\]
$\mathcal{L}_{\text{EO}}$ enforces equalized odds using differentiable soft rates:
\[
    \mathcal{L}_{\text{EO}} = (\widehat{\text{TPR}}_0 - \widehat{\text{TPR}}_1)^2 + (\widehat{\text{FPR}}_0 - \widehat{\text{FPR}}_1)^2
\]
where $\widehat{\text{TPR}}_g = \frac{\sum_{i: y_i=1, a_i=g} \sigma(f_\theta(x_i))}{\sum_{i: y_i=1, a_i=g} 1}$ and $\widehat{\text{FPR}}_g = \frac{\sum_{i: y_i=0, a_i=g} \sigma(f_\theta(x_i))}{\sum_{i: y_i=0, a_i=g} 1}$ replace hard predictions with sigmoid outputs $\sigma(f_\theta(x))$ to ensure differentiability. And $\mathcal{L}_{\text{CEC}}$ enforces explanation consistency:
\[
    \mathcal{L}_{\text{CEC}} = \frac{1}{m}\sum_{i=1}^m \left(\frac{\left\|\text{IG}_{\text{norm}}(x_i; b_{y_i,a_i}) - \text{IG}_{\text{norm}}(\tilde{x}_i; b_{y_i,a_i})\right\|_2}{2}\right)^2
\]
We square the normalized CEC scores to obtain a differentiable objective that penalizes large deviations more heavily than small ones. Since individual CEC scores lie in $[0,1]$, the loss is bounded in $[0,1]$, providing a consistent scale relative to the other loss terms.
 
The three terms balance utility (accurate predictions), outcome fairness (equalized error rates), and procedural fairness (consistent reasoning). An important distinction is that these are not redundant because $\mathcal{L}_{\text{EO}} \approx 0$ does not imply $\mathcal{L}_{\text{CEC}} \approx 0$. Models can achieve perfect equalized odds while using entirely different feature weightings for different groups (Regime~B in our taxonomy). Conversely, explanation consistency alone does not guarantee outcome fairness. The joint objective ensures both are satisfied simultaneously. The CEC loss is differentiable through the integrated gradients computation via standard backpropagation, and its $[0,1]$ range ensures stable gradient magnitudes relative to the other loss components.
 
\subsection{Training Algorithm}
Algorithm \ref{alg:training} describes the complete procedure. Phase 1 preprocessing (computing baselines and counterfactuals) runs once before training. Phase 2 shows the training, then proceeds via standard minibatch SGD with the augmented loss. 
 
\begin{algorithm}[t]
\caption{CEC Training}
\label{alg:training}
\small
\begin{algorithmic}[1]
\REQUIRE Data $\mathcal{D}$, financial features $\mathcal{F}$, $\lambda_{\text{EO}}, \lambda_{\text{CEC}}, \eta, E, T$
\ENSURE Trained model $f_\theta$
\STATE \textbf{// Phase 1: Preprocessing}
\STATE $\{b_{y,a}\}_{y,a \in \{0,1\}} \gets$ label-group means from $\mathcal{D}$
\STATE $\{\tilde{x}_i\} \gets$ Algorithm~\ref{alg:cf_generation} on $\mathcal{D}$
\STATE Initialize parameters $\theta$
\STATE \textbf{// Phase 2: Training}
\FOR{epoch $= 1, \ldots, E$}
    \FOR{each minibatch $\mathcal{B} = \{(x_i, y_i, a_i, \tilde{x}_i)\}$}
        \STATE $\mathcal{L}_{\text{pred}} \gets$ BCE$(f_\theta, \mathcal{B})$
        \STATE $\mathcal{L}_{\text{EO}} \gets (\text{TPR}_0 {-} \text{TPR}_1)^2 + (\text{FPR}_0 {-} \text{FPR}_1)^2$
        \STATE $\mathcal{L}_{\text{CEC}} \gets 0$
        \FOR{$(x_i, \tilde{x}_i) \in \mathcal{B}$}
            \STATE $g_i \gets \text{IG}(f_\theta, x_i, b_{y_i,a_i}, T)$
            \STATE $\tilde{g}_i \gets \text{IG}(f_\theta, \tilde{x}_i, b_{y_i,a_i}, T)$ \COMMENT{Same $b$}
            \STATE $g_i \gets g_i / (\|g_i\| + \epsilon)$; \; $\tilde{g}_i \gets \tilde{g}_i / (\|\tilde{g}_i\| + \epsilon)$
            \STATE $\mathcal{L}_{\text{CEC}} \mathrel{+}= (\|g_i - \tilde{g}_i\|_2 \;/\; 2)^2$
        \ENDFOR
        \STATE $\mathcal{L}_{\text{CEC}} \gets \mathcal{L}_{\text{CEC}} / |\mathcal{B}|$
        \STATE $\theta \gets \theta - \eta\,\nabla_\theta(\mathcal{L}_{\text{pred}} + \lambda_{\text{EO}}\mathcal{L}_{\text{EO}} + \lambda_{\text{CEC}}\mathcal{L}_{\text{CEC}})$
    \ENDFOR
\ENDFOR
\end{algorithmic}
\end{algorithm}

\subsection{Computational Complexity.}
CEC introduces overhead in two phases. In Phase~1, computing label-group baselines costs $O(nd)$ where $n$ is the dataset size and $d$ the feature dimensionality. Counterfactual generation via Algorithm~\ref{alg:cf_generation} requires building four KD-tree indices at $O(n \log n \cdot d_\mathcal{F})$ each, where $d_\mathcal{F} = |\mathcal{F}|$, followed by $n$ nearest-neighbor queries at $O(d_\mathcal{F} \log n)$ per query, yielding a total preprocessing cost of $O(n \log n \cdot d_\mathcal{F})$. In Phase~2, each training batch incurs the standard forward-pass cost of $O(md)$ for prediction and equalized odds losses, where $m$ is the batch size. The CEC loss dominates: computing integrated gradients for both factual and counterfactual requires $2mT$ forward passes through the network, where $T$ is the number of integration steps, giving a per-batch CEC cost of $O(mTd)$. The total per-epoch cost is therefore $O\!\left(\frac{n}{m} \cdot mTd\right) = O(nTd)$, which is comparable to $O(nd)$ for standard training and an additional overhead factor of $T$. With GPU parallelism of the integration steps and shared backward pass costs, the computation cost can be lowered by $3\times$ to $5 \times$ rather than the na\"{i}ve $32\times$

\section{Experimental Setting}
Our experiments are designed to answer four  questions:
\begin{itemize}
    \item \textbf{RQ1:} Can models satisfy outcome fairness (equalized odds) while exhibiting hidden procedural bias?
    
    \item \textbf{RQ2:} Does the CEC framework effectively detect and mitigate hidden procedural bias while maintaining outcome fairness and predictive accuracy?
    
    \item \textbf{RQ3:} How do existing fairness interventions perform with respect to procedural fairness metrics?
    
    \item \textbf{RQ4:} Does the effectiveness of CEC generalize across synthetic data, benchmark datasets, and real-world lending data?
\end{itemize}
\subsection{Baselines}
We evaluated CEC against six baselines across four datasets of increasing complexity and realism. We compared against representative methods from each fairness paradigm: \textbf{Unconstrained} (standard neural network, no fairness constraint), \textbf{Disparate Impact Remover} (DIR) \cite{feldman2015certifying} for pre-processing, \textbf{Hardt Post-Processing} \cite{hardt2016equality} for post-processing, \textbf{Agarwal Reductions} \cite{agarwal2018reductions} and \textbf{Adversarial Debiasing} \cite{zhang2018mitigating} for in-processing, and \textbf{Lagrangian Fair Learning} \cite{cotter2019optimization} for constrained optimization. All methods target equalized odds constraints for direct comparability with CEC.

\subsection{Evaluation Metrics}
 We evaluated models along three dimensions using complementary metrics.
 \begin{itemize}
     \item \textbf{Utility:} F1 score (harmonic mean of precision and recall) and AUC (area under the ROC curve) measure classification quality.
     \item \textbf{Outcome Fairness:} Equalized odds gap $= \max(|\text{TPR}_0 - \text{TPR}_1|, |\text{FPR}_0 - \text{FPR}_1|)$ and statistical parity gap $= |P(\hat{Y}{=}1|A{=}0) - P(\hat{Y}{=}1|A{=}1)|$ capture group-level outcome disparities.
     \item \textbf{Procedural Fairness:} The \emph{CEC score} (Definition~\ref{def:cec}) measures average explanation consistency. The \emph{Prediction Flip Rate} (PFR) $= P(\hat{y}(x) \neq \hat{y}(\tilde{x}))$ captures individual-level outcome instability.
     The \emph{regime distribution} reports the fraction of test examples in each of the four regimes (A--D) from Table~\ref{tab:taxonomy}, with particular emphasis on Regime~B (hidden bias).
 \end{itemize}
PFR captures complementary aspects of fairness by detecting outcome-level instability (Regimes~C and~D). A model with low PFR but high CEC is precisely one that exhibits hidden procedural bias, which is the primary target of our framework.

\subsection{Datasets}
As mentioned, we evaluated our method on four different data sets.

\textbf{Synthetic Data.} We generated $n{=}10{,}000$ samples with $d{=}20$ features partitioned into financial ($|\mathcal{F}|{=}10$), proxy ($|\mathcal{P}|{=}5$), and noise ($|\mathcal{N}|{=}5$) features. Financial features were generated independently of the protected attribute $A$; proxy features were shifted by $0.5\sigma$ between groups to simulate real-world correlations (e.g., residential segregation). Ground-truth labels depended only on financial features: $y = \mathbb{I}[\sum_{j \in \mathcal{F}} w_j x_j + \epsilon > \tau]$. 

\textbf{German Credit} \cite{statlog_german_credit_data_144}. This data set is a standard fairness benchmark with 1,000 instances, 20 features, and binary creditworthiness labels. We used gender as the protected attribute and defined $\mathcal{F}$ to include credit amount, duration, installment rate, present residence, age, number of existing credits, and number of dependents. Features encoding gender (personal status) and immigration status were excluded from $\mathcal{F}$.

\textbf{Adult Income} \cite{adult_2}. Derived from the 1994 U.S.\ Census, this dataset contains 48,842 instances with demographic and employment features, and a binary label indicating whether annual income exceeds $\$50K.$ We used gender as the protected attribute. Financial features included capital-gain, capital-loss, hours-per-week, and occupation category. Education level and marital status were excluded from $\mathcal{F}$ as potential demographic proxies.

\textbf{HMDA Mortgage Data} \cite{hmda_data}. We used 2024 Home Mortgage Disclosure Act data from the Consumer Financial Protection Bureau, which provides real-world mortgage lending records with regulatory-grade detail. The protected attribute was race (White vs.\ non-White). Financial features included loan amount, income, debt-to-income ratio, property value, loan-to-value ratio, and loan term. We excluded census tract, county code, and applicant ethnicity from $\mathcal{F}$. This dataset provides the most realistic evaluation, reflecting actual lending patterns subject to federal fair lending oversight.

\subsection{Implementation Details}
All methods used neural networks with the same architecture, which was tuned to two hidden layers of sizes $[128, 64]$ with ReLU activations and dropout rate $0.2$. We trained for 30 epochs using the Adam optimizer with learning rate $3 \times 10^{-4}$ and batch size 64. For CEC, we set $\lambda_{\text{EO}} = 1.0$, $\lambda_{\text{CEC}} = 1.0$, and $T{=}32$ IG integration steps. We performed 5-fold cross-validation and report means and standard deviations across folds.

\begin{table*}[t]
\centering
\caption{Performance comparison across four datasets. Results are mean $\pm$ std. $\downarrow$ indicates lower is better (EO, SP, CEC), and $\uparrow$ indicates higher is better (AUC, F1). \textbf{Bold} indicates best performance.}
\label{tab:main_results}
\resizebox{\textwidth}{!}{%
\begin{tabular}{llccccccc}
\toprule
\textbf{Dataset} & \textbf{Metric} & \textbf{Unconstrained} & \textbf{DIR} & \textbf{Hardt} & \textbf{Agarwal} & \textbf{Lagrangian} & \textbf{Adversarial} & \textbf{\textsc{CEC (MODEL) }} \\
\midrule

\multirow{5}{*}{German Credit}
& AUC$\uparrow$ & 0.735$\pm$0.032 & 0.512$\pm$0.085 & 0.569$\pm$0.068 & 0.620$\pm$0.038 & \textbf{0.744$\pm$0.028} & 0.731$\pm$0.019 & 0.693$\pm$0.091 \\
& F1$\uparrow$  & \textbf{0.826$\pm$0.008} & 0.579$\pm$0.060 & 0.622$\pm$0.040 & 0.813$\pm$0.023 & \textbf{0.826$\pm$0.007} & 0.818$\pm$0.009 & 0.815$\pm$0.017 \\
& EO$\downarrow$ & 0.210$\pm$0.110 & 0.339$\pm$0.139 & 0.182$\pm$0.102 & 0.212$\pm$0.061 & 0.252$\pm$0.121 & 0.135$\pm$0.057 & \textbf{0.018$\pm$0.020} \\
& SP$\downarrow$ & 0.085$\pm$0.045 & 0.166$\pm$0.066 & 0.062$\pm$0.035 & 0.083$\pm$0.044 & 0.108$\pm$0.032 & 0.038$\pm$0.040 & \textbf{0.009$\pm$0.011} \\
& CEC$\downarrow$ & 0.559$\pm$0.016 & 0.536$\pm$0.007 & 0.547$\pm$0.018 & 0.544$\pm$0.014 & 0.561$\pm$0.007 & 0.556$\pm$0.008 & \textbf{0.208$\pm$0.020} \\
& Reg.~B (\%)$\downarrow$ & 67.9$\pm$2.1 & 90.7$\pm$2.6 & 91.6$\pm$1.8 & 89.9$\pm$2.6 & 68.0$\pm$2.9 & 71.4$\pm$2.7 & \textbf{17.1$\pm$6.6} \\
\midrule

\multirow{5}{*}{Adult Income}
& AUC$\uparrow$ & \textbf{0.907$\pm$0.002} & 0.734$\pm$0.048 & 0.771$\pm$0.015 & 0.742$\pm$0.007 & \textbf{0.907$\pm$0.002} & 0.895$\pm$0.001 & 0.876$\pm$0.009 \\
& F1$\uparrow$  & 0.678$\pm$0.005 & 0.505$\pm$0.034 & 0.534$\pm$0.016 & 0.627$\pm$0.010 & \textbf{0.682$\pm$0.009} & 0.657$\pm$0.004 & 0.656$\pm$0.066 \\
& EO$\downarrow$ & 0.124$\pm$0.030 & 0.297$\pm$0.265 & 0.351$\pm$0.125 & \textbf{0.048$\pm$0.029} & 0.140$\pm$0.059 & 0.158$\pm$0.024 & 0.062$\pm$0.036 \\
& SP$\downarrow$ & 0.180$\pm$0.016 & 0.203$\pm$0.100 & 0.276$\pm$0.037 & 0.112$\pm$0.014 & 0.185$\pm$0.018 & 0.116$\pm$0.017 & \textbf{0.083$\pm$0.019} \\
& CEC$\downarrow$ & 0.478$\pm$0.008 & 0.582$\pm$0.008 & 0.602$\pm$0.005 & 0.600$\pm$0.009 & 0.479$\pm$0.006 & 0.485$\pm$0.002 & \textbf{0.083$\pm$0.002} \\
& Reg.~B (\%)$\downarrow$ & 67.8$\pm$2.6 & 91.4$\pm$1.3 & 92.1$\pm$1.1 & 93.4$\pm$0.3 & 66.9$\pm$1.5 & 67.2$\pm$1.9 & \textbf{1.0$\pm$0.2} \\
\midrule

\multirow{5}{*}{Synthetic}
& AUC$\uparrow$ & {0.992$\pm$0.002} & 0.634$\pm$0.029 & 0.582$\pm$0.064 & 0.970$\pm$0.007 & {0.992$\pm$0.002} & {0.992$\pm$0.002} & \textbf{0.996$\pm$0.005} \\
& F1$\uparrow$  & 0.937$\pm$0.009 & 0.552$\pm$0.026 & 0.501$\pm$0.043 & \textbf{0.964$\pm$0.008} & 0.934$\pm$0.010 & 0.939$\pm$0.008 & 0.951$\pm$0.010 \\
& EO$\downarrow$ & 0.049$\pm$0.028 & 0.065$\pm$0.015 & 0.596$\pm$0.145 & 0.023$\pm$0.011 & 0.052$\pm$0.026 & 0.026$\pm$0.015 & \textbf{0.019$\pm$0.012} \\
& SP$\downarrow$ & 0.035$\pm$0.008 & 0.024$\pm$0.014 & 0.651$\pm$0.072 & \textbf{0.016$\pm$0.021} & 0.036$\pm$0.015 & 0.023$\pm$0.015 & 0.021$\pm$0.021 \\
& CEC$\downarrow$ & 0.293$\pm$0.003 & 0.612$\pm$0.034 & 0.601$\pm$0.027 & 0.624$\pm$0.016 & 0.293$\pm$0.001 & 0.292$\pm$0.001 & \textbf{0.275$\pm$0.003} \\
& Reg.~B (\%)$\downarrow$ & 38.9$\pm$0.9 & 98.7$\pm$1.1 & 98.5$\pm$0.9 & 98.5$\pm$1.5 & 40.2$\pm$1.4 & 39.1$\pm$1.0 & \textbf{30.0$\pm$1.5} \\
\midrule

\multirow{5}{*}{HMDA}
& AUC$\uparrow$ & \textbf{0.998$\pm$0.001} & 0.817$\pm$0.016 & 0.805$\pm$0.023 & 0.744$\pm$0.007 & \textbf{0.998$\pm$0.001} & \textbf{0.998$\pm$0.001} & 0.995$\pm$0.078 \\
& F1$\uparrow$  & 0.989$\pm$0.003 & 0.765$\pm$0.018 & 0.764$\pm$0.018 & 0.756$\pm$0.008 & \textbf{0.990$\pm$0.003} & \textbf{0.990$\pm$0.002} & 0.968$\pm$0.090 \\
& EO$\downarrow$ & 0.006$\pm$0.003 & 0.293$\pm$0.194 & 0.065$\pm$0.011 & 0.045$\pm$0.022 & \textbf{0.005$\pm$0.001} & \textbf{0.005$\pm$0.001} & 0.008$\pm$0.007 \\
& SP$\downarrow$ & 0.028$\pm$0.014 & 0.154$\pm$0.108 & 0.020$\pm$0.009 & 0.017$\pm$0.010 & 0.027$\pm$0.015 & 0.027$\pm$0.014 & \textbf{0.016$\pm$0.013} \\
& CEC$\downarrow$ & 0.360$\pm$0.007 & 0.389$\pm$0.016 & 0.398$\pm$0.029 & 0.394$\pm$0.037 & 0.362$\pm$0.009 & 0.359$\pm$0.011 & \textbf{0.159$\pm$0.011} \\
& Reg.~B (\%)$\downarrow$ & 50.7$\pm$0.8 & 52.0$\pm$3.3 & 49.7$\pm$2.4 & 50.8$\pm$1.2 & 51.5$\pm$1.7 & 52.7$\pm$2.4 & \textbf{1.0$\pm$2.0} \\
\bottomrule
\end{tabular}%
}
\end{table*}

\section{Experimental Results}
The results in Table~\ref{tab:main_results} show a clear disconnect between outcome fairness and procedural fairness across all four datasets. Several baseline methods, such as Adversarial Debiasing and Agarwal Reductions, achieved competitive or even superior equalized odds and statistical parity scores. Particularly, these methods substantially reduced outcome disparities on multiple datasets. However, these same methods exhibited CEC scores that are comparable to, or in some cases \emph{worse than}, the unconstrained baseline. This confirms the motivation of our work, which is that optimizing for outcome fairness does not address hidden procedural bias. The CEC-trained model is the only method that consistently achieved the lowest CEC scores across all four datasets, sometimes by a wide margin, while simultaneously delivering competitive equalized odds and statistical parity performance. Also, the CEC-trained model did so with only modest reductions in predictive utility, maintaining F1 scores within a few percentage points of the best-performing baselines on German Credit and Adult Income. On Synthetic and HMDA, the higher variance in CEC's utility metrics reflects the sensitivity of the multi-objective loss at the chosen $\lambda$ configuration.

The magnitude of CEC's improvement varied meaningfully across datasets, and these differences reflect the underlying structure of procedural bias in each dataset. On Synthetic data, where the ground-truth labeling function uses only financial features and the procedural bias is embedded synthetically, all methods started with relatively low CEC scores, and the gap between the CEC model and baselines was correspondingly modest. This is expected because the synthetic design isolates procedural bias in a controlled environment rather than in the trained models themselves. The picture changed dramatically on German Credit, a dataset collected in 1990s Germany when gender-differentiated lending practices were more prevalent and structurally embedded. In this dataset, all baseline methods exhibited CEC scores above 0.53, while the CEC-model reduced this to 0.21. The historical context matters because features like employment tenure carry a gender-correlated signal from an era of systematic labor market segregation, creating precisely the kind of hidden procedural pathways that outcome-based methods cannot detect. Adult Income showed a similar pattern, with the CEC-model achieving an even more significant reduction from 0.60 to 0.08, reflecting the well-documented structural entanglement of race and gender with income-determining features in U.S. labor data in the 1990s. On HMDA, which represents recent mortgage lending under active ECOA oversight, baseline CEC scores clustered around 0.36--0.40, and the CEC model reduced this to 0.26. The smaller relative improvement is expected because modern regulatory pressure has likely reduced the most overt forms of procedural bias in mortgage underwriting, yet a meaningful procedural gap persisted in these experiments.

\subsection{PFR-CEC Space}

\begin{figure*}[t]
\centering
\includegraphics[width=0.98\columnwidth]{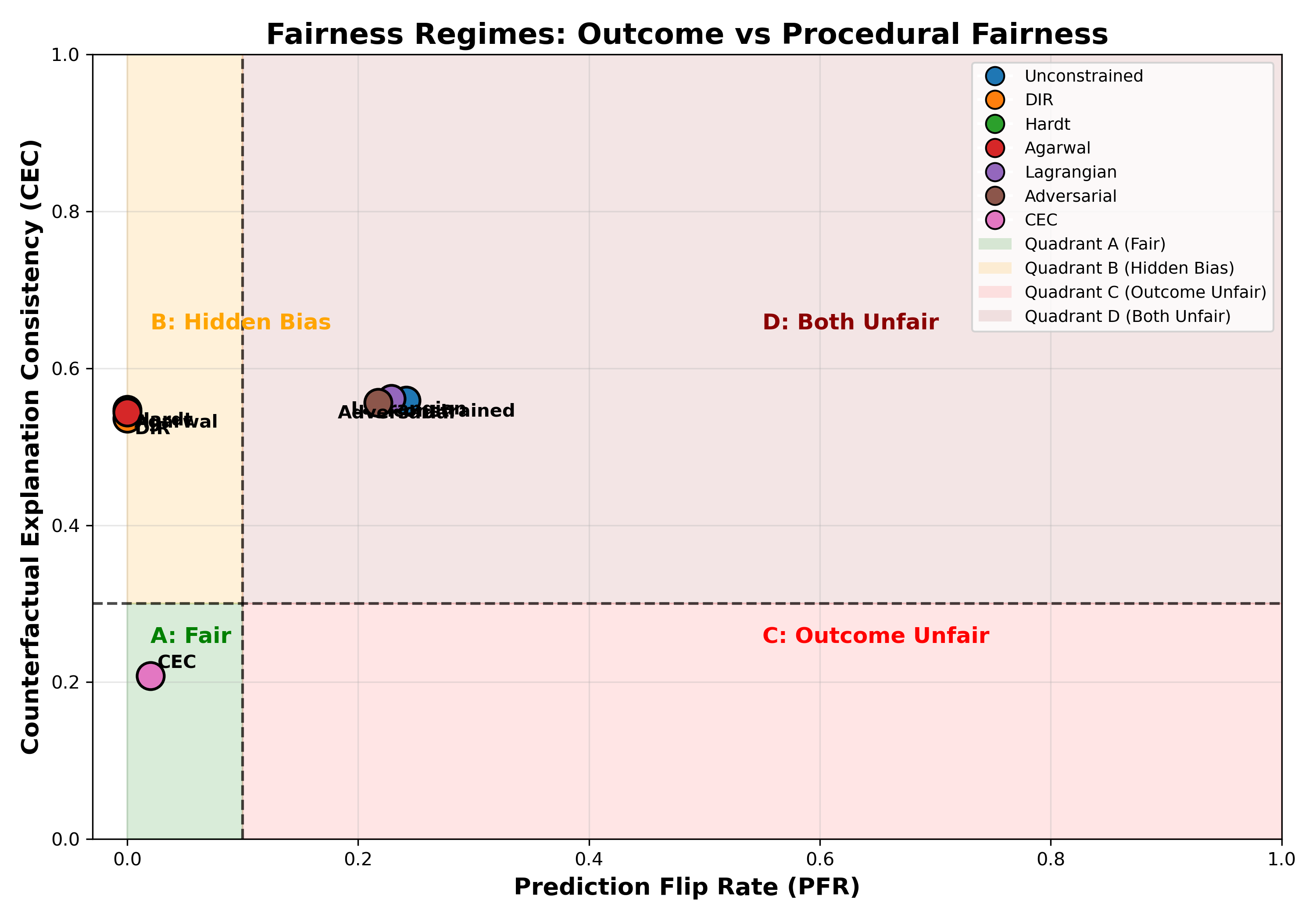}
\includegraphics[width=0.98\columnwidth]{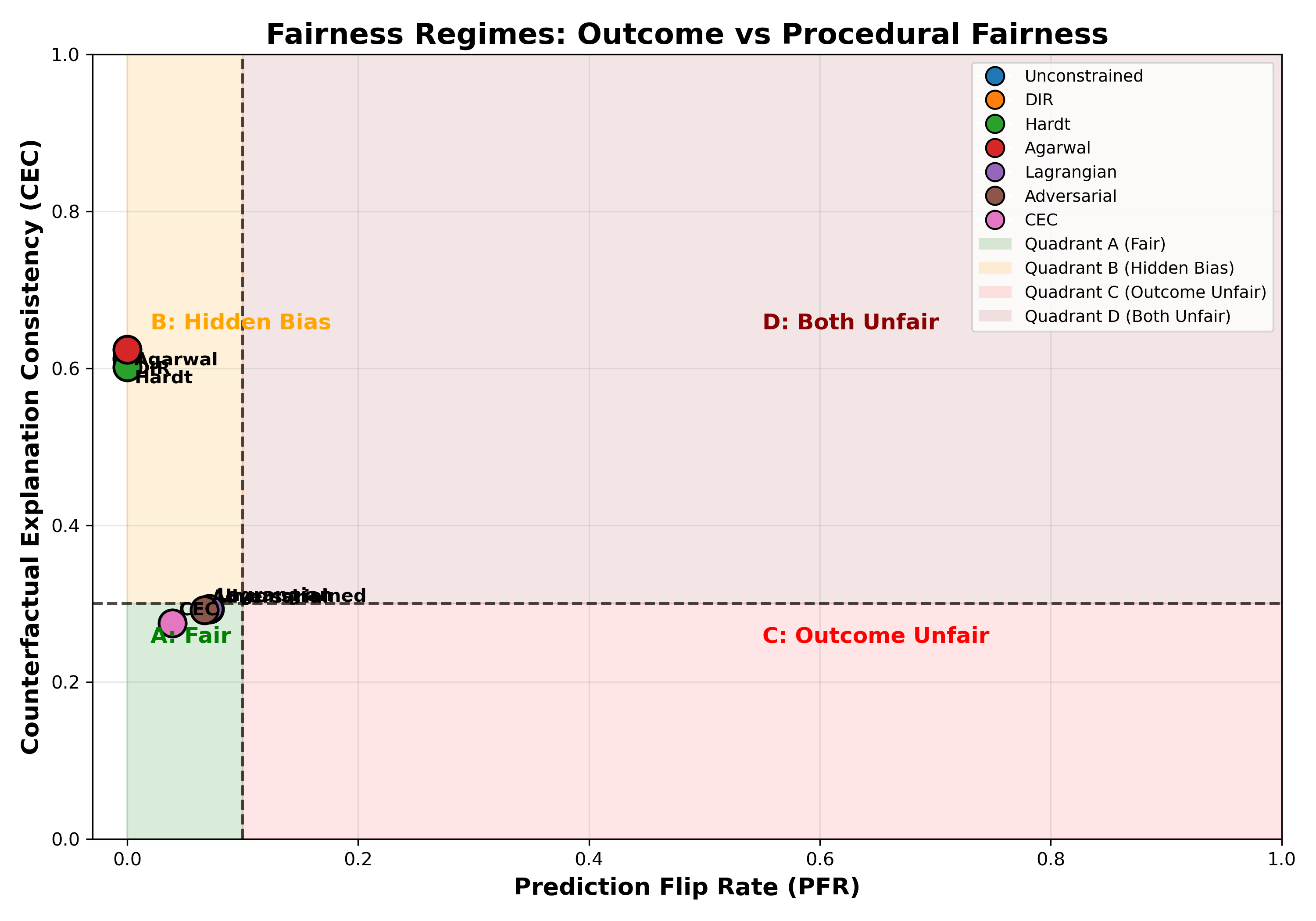}
\caption{Fairness regime plots for German Credit (left) and Synthetic (right). Each method is positioned by its Prediction Flip Rate (PFR) and (CEC). CEC-model is the only method consistently in Quadrant A on both datasets.}
\label{fig:regimes}
\end{figure*}

Figure \ref{fig:regimes} visualizes the fairness regime landscape for the German Credit and Synthetic datasets by plotting each method in PFR--CEC space, where the four quadrants correspond to the regimes in Table \ref{tab:taxonomy}. On German Credit, the methods are split into several clusters. DIR, Hardt, and Agarwal achieve low prediction flip rates but exhibit high CEC, placing them in Quadrant B, which means they equalized predictions across counterfactual pairs while relying on fundamentally different reasoning, demonstrating the hidden procedural bias that motivates our work. Unconstrained, Lagrangian, and Adversarial showed worse results. Their high PFR combined with high CEC positions them in Quadrant D, where both outcome and procedural fairness are violated. The CEC model was the only method that reached Quadrant A, achieving low scores on both axes simultaneously. 

The Synthetic dataset shows a slightly different result. Here, the Unconstrained, Lagrangian, and Adversarial methods sit near the boundary between Quadrants A and B, reflecting the moderate procedural bias embedded by design. Meanwhile, DIR, Hardt, and Agarwal are pushed into Quadrant B with higher CEC, suggesting that their fairness interventions inadvertently \emph{increased} explanation instability even as they equalized predictions. The CEC model occupies the lowest CEC position. Taken together, the two plots illustrate that Quadrant B is a deficiency that most previous fairness algorithms lack, and the CEC model is the only method among those compared that systematically moves models toward Quadrant A.

\subsection{Regime Distribution}
\begin{figure*}[t]
\centering
\includegraphics[width=0.98\columnwidth]{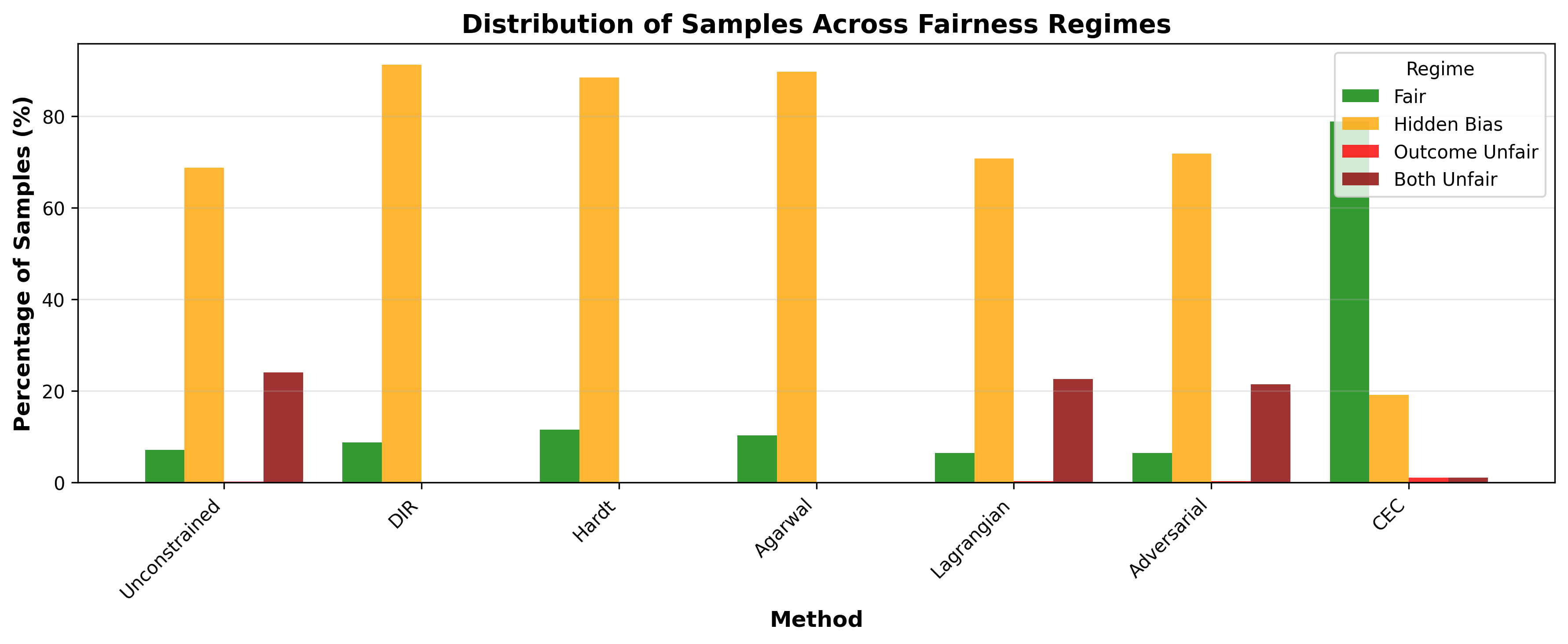}
\includegraphics[width=0.98\columnwidth]{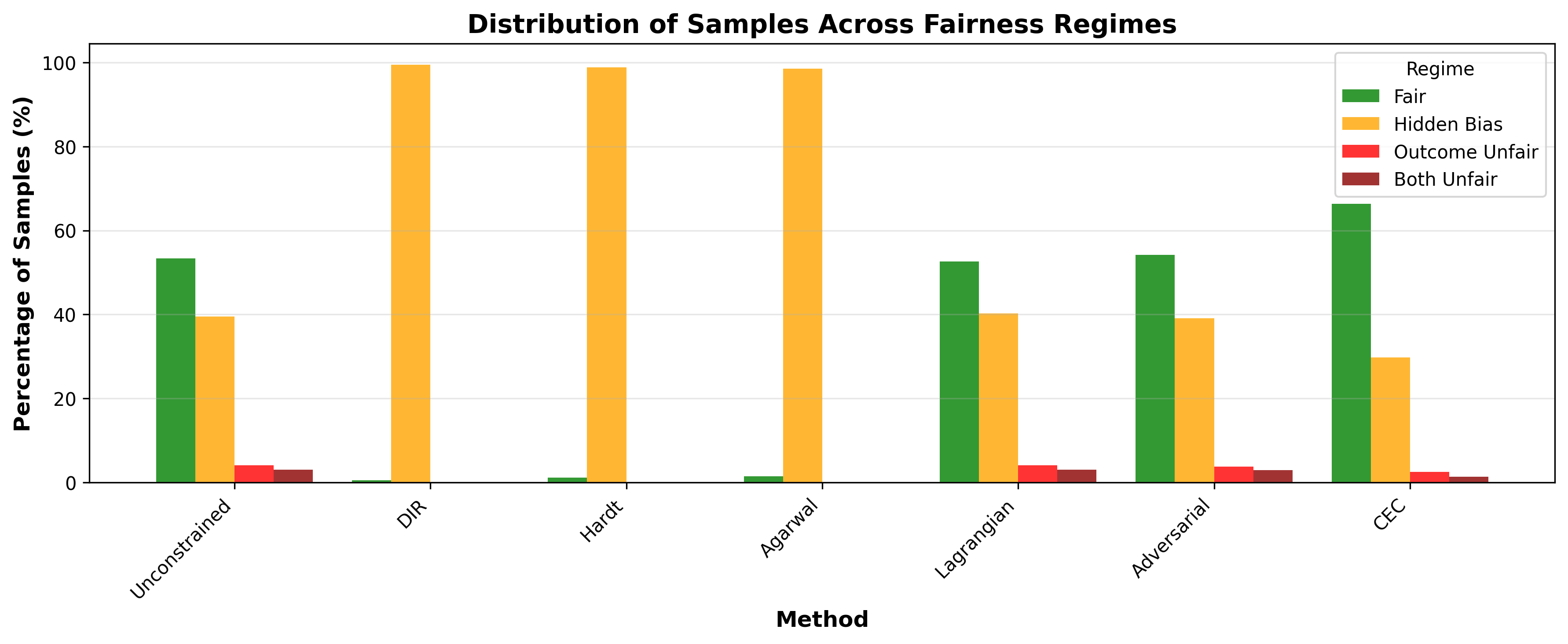}
\caption{Regime distribution for German Credit (left) and Synthetic (right). Each method shows the percentage of its sample that falls into four taxonomy categories.}
\label{fig:regimes_dist}
\end{figure*}

Figure~\ref{fig:regimes_dist} provides a complementary per-sample view by showing the fraction of test individuals assigned to each fairness regime. On German Credit, the dominance of Regime B is striking in DIR, Hardt, and Agarwal methods, with approximately 90\% of all individuals in the hidden bias regime, meaning that for nearly every applicant, the model produced the correct prediction but arrived at it through group-dependent reasoning. Even the Unconstrained model assigned around 69\% of samples to Regime B, with an additional 24\% in Regime D (both unfair). The CEC model assigned nearly 80\% of samples into Regime A (fully fair) while reducing Regime B to under 19\%. 

On Synthetic data, the pattern is similar but less extreme. Most baselines placed 39--99\% of samples in Regime B, depending on the method, while the CEC model reduced this to approximately 30\% and raised Regime A to 66\%. A notable observation is that outcome-focused methods like DIR, Hardt, and Agarwal actually \emph{worsened} the regime distribution compared to the Unconstrained baseline. This happened because these methods force predictions to agree across groups without constraining explanations by converting Regime D samples (which are mostly detectable by outcome metrics) into Regime B samples (which are not detectable by outcome metrics). 
This conversion effect is why outcome-oriented fairness metrics cannot mitigate procedural bias.

\subsection{Trade-Off Analysis}
\begin{table}[t]
\centering
\small
\caption{Pareto non-dominance across 5 folds $\times$ 4 datasets. 
We note that CEC model has a structural advantage over baselines due to its multi-objective nature.}
\label{tab:pareto}
\begin{tabular}{@{}lccccc@{}}
\toprule
\textbf{Method} & \textbf{German} & \textbf{Adult} & \textbf{Synthetic} & \textbf{HMDA} & \textbf{Total} \\
\midrule
Unconstrained & 1/5 & 5/5 & 2/5 & 5/5 & 13/20 \\
DIR & 0/5 & 0/5 & 0/5 & 0/5 & 0/20 \\
Hardt & 0/5 & 1/5 & 0/5 & 0/5 & 1/20 \\
Agarwal & 2/5 & 5/5 & 5/5 & 0/5 & 12/20 \\
Lagrangian & 1/5 & 4/5 & 2/5 & 4/5 & 11/20 \\
Adversarial & 2/5 & 3/5 & 2/5 & 5/5 & 12/20 \\
\textbf{\textsc{CEC (Ours)}} & 5/5 & 5/5 & 5/5 & 5/5 & \textbf{20/20} \\
\bottomrule
\end{tabular}
\end{table}

To assess whether procedural fairness comes at the cost of utility or outcome fairness, we compute Pareto non-dominance across all 5 folds $\times$ 4 datasets. A method is non-dominated on a given fold if no other method achieves strictly better F1, EO gap, and CEC score simultaneously. Table~\ref{tab:pareto} reports the results. The CEC model is the only method that produced a non-dominated solution on every fold of every dataset (20/20), meaning that no baseline ever jointly outperformed the CEC model across all three objectives. On the other hand, the next-best methods are Unconstrained (13/20), Agarwal (12/20), and Adversarial (12/20), which were frequently dominated because their strong utility came paired with poor CEC scores that CEC-model improved upon without sacrificing the other dimensions. DIR and Hardt were dominated on nearly every fold (0/20 and 1/20, respectively), reflecting their tendency to harm both utility and procedural fairness because they did not mitigate bias in modeling, which is the most crucial part. The result was especially notable on German Credit, where the CEC model was non-dominated on all 5 folds, while Unconstrained, Lagrangian, and Hardt were each non-dominated on at most 1 fold, which confirms that on datasets with substantial hidden procedural bias, the CEC model achieved trade-offs that no other method could match.

\section{Ablation Study}
\subsection{Hyperparameter Sensitivity}
\begin{figure*}[t]
    \centering
    \includegraphics[width=0.99\linewidth]{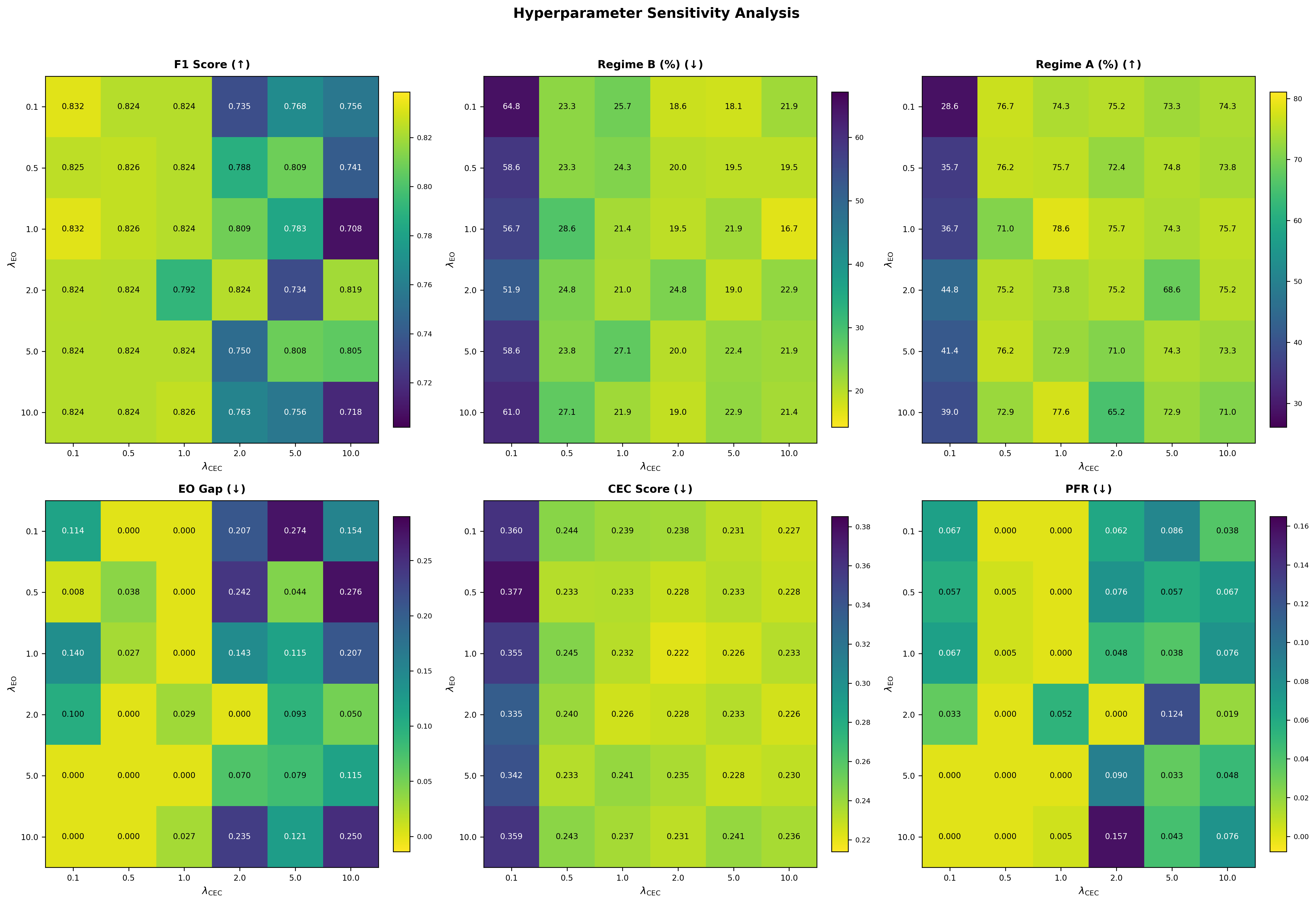}
    \caption{Sensitivity Plots of $\lambda_{EO}$ and $\lambda_{CEC}$ on German Dataset}
    \label{fig:heatmap}
\end{figure*}

Figure~\ref{fig:heatmap} presents a hyperparameter sensitivity analysis on German Credit, sweeping $\lambda_{\text{CEC}}$ and $\lambda_{\text{EO}}$ over $\{0.1, 0.5, 1.0, 2.0, 5.0, 10.0\}$ across six metrics. The heatmaps reveal that CEC responded primarily to $\lambda_{\text{CEC}}$: dropped sharply between $\lambda_{\text{CEC}} = 0.1$ and $\lambda_{\text{CEC}} = 0.5$ and plateaued thereafter, indicating that even modest procedural fairness pressure yielded substantial gains. Regime B and Regime A followed the same but alternating pattern, falling from over 50\% at $\lambda_{\text{CEC}} = 0.1$ to approximately 20\% for $\lambda_{\text{CEC}} \geq 1.0$, largely independent of $\lambda_{\text{EO}}$. F1 was remarkably stable across the moderate regularization region ($\lambda_{\text{CEC}} \leq 1.0$), remaining above 0.82 regardless of $\lambda_{\text{EO}}$, with meaningful degradation occurring only at $\lambda_{\text{CEC}} \geq 2.0$. The EO gap exhibited more complex behavior: it was reliably near zero in the low-to-moderate $\lambda_{\text{CEC}}$ region ($\leq 1.0$) but became erratic at higher values, where aggressive CEC regularization pushed the model toward near-constant predictions that destabilized group-wise rate estimates. PFR remained low throughout, confirming that prediction stability was preserved even under strong regularization. Taken together, the heatmaps identify a favorable operating region at $\lambda_{\text{CEC}} \in [0.5, 1.0]$ and $\lambda_{\text{EO}} \leq 2.0$ where procedural fairness improved substantially, outcome fairness was maintained, and utility cost was negligible. This sensitivity analysis confirms that our default configuration ($\lambda_{\text{CEC}} = 1.0$, $\lambda_{\text{EO}} = 1.0$) was well-situated within an optimal region.

\subsection{Effect of Composite Loss Function}
\begin{table*}[t]
    \centering
    \caption{Performance of different parts of the loss function}
    \label{tab:ablation}
    \begin{tabular}{llllll}
\toprule
Variant & F1 & EOD & CEC &  regime\_B \\
\midrule
Pred Only & 0.8205 & 0.2615 & 0.5784 &  70.4762 \\
Pred + EO & 0.8401 & 0.3460 & 0.5700 &  71.4286 \\
Pred + CEC & 0.8235 & 0.0000 & 0.2329 &  20.9524 \\
Full (Pred + EO + CEC) & 0.8352 & 0.0696 &  0.2360 & 23.3333 \\
\bottomrule
\end{tabular}
\end{table*}

To isolate the contribution of each loss component, Table~\ref{tab:ablation} reports results for four training variants on German Credit: prediction loss only, prediction with equalized odds, prediction with CEC, and the full objective. The most striking finding is that adding EO regularization alone (Pred + EO) reduced neither the CEC score nor the Regime B fraction compared to the prediction-only baseline, in fact, Regime B slightly \emph{increased} from 70.5\% to 71.4\%, confirming that outcome fairness constraints are orthogonal to procedural fairness and can even marginally exacerbate hidden bias. Conversely, adding CEC regularization alone (Pred + CEC) dramatically reduced the CEC score from 0.578 to 0.233 and Regime~B from 70.5\% to 21.0\%, while also driving the EO gap to zero as a side effect on German Credit data, suggesting that enforcing consistent reasoning across groups naturally promotes outcome equity even without explicit EO constraints. The full model (Pred + EO + CEC) achieved the best F1 while maintaining strong CEC and EO performance, with a slight increase in EO gap relative to the CEC-only variant as the three objectives negotiated their trade-off. These results demonstrate that CEC is the essential component for procedural fairness and that its benefits are largely complementary to, rather than redundant with, outcome fairness regularization.

\begin{figure}[t]
    \centering
    \includegraphics[width=0.99\linewidth]{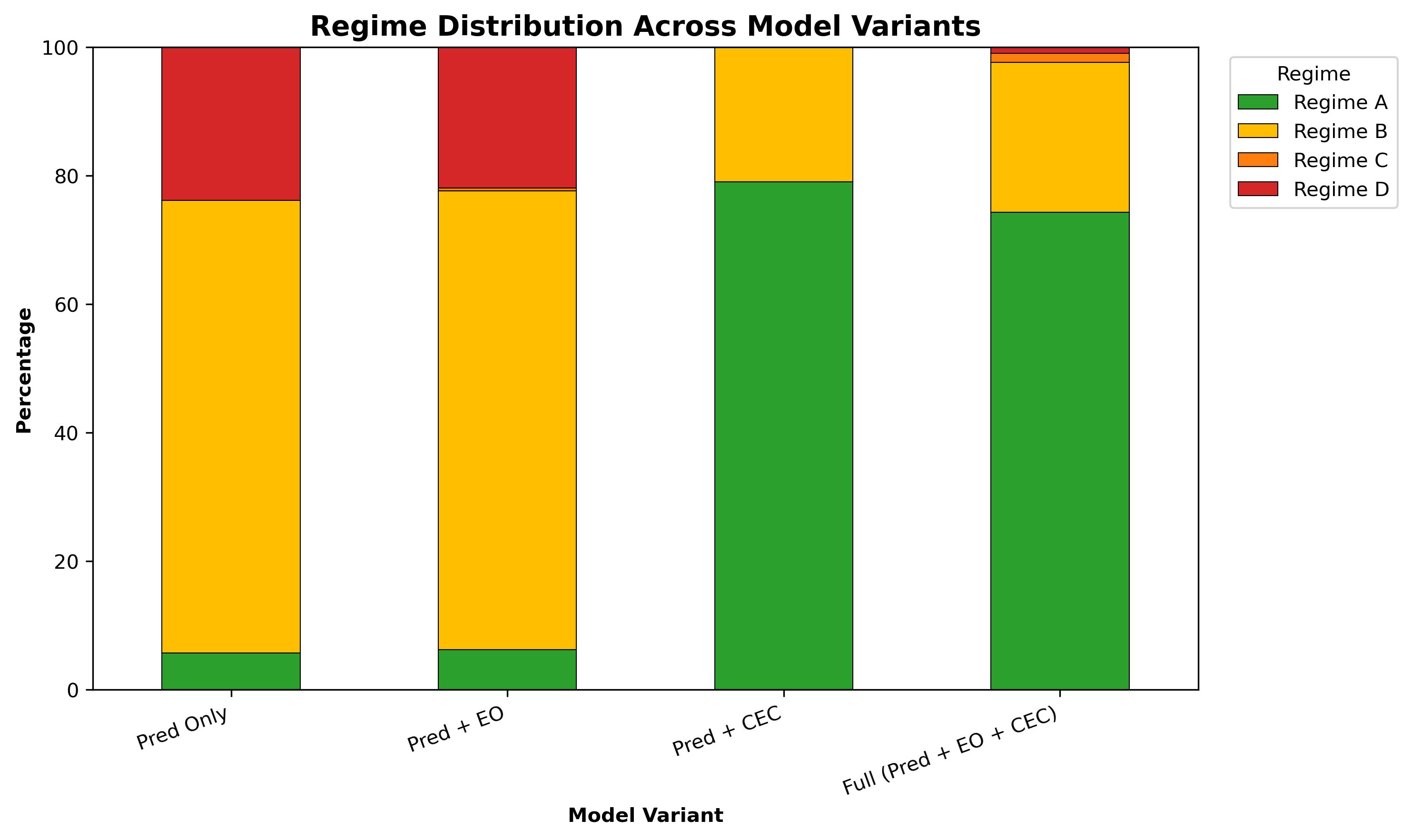}
    \caption{Regime distribution of each part of the loss function.}
    \label{fig:regime_transitions}
\end{figure}

Figure \ref{fig:regime_transitions} visualizes these transitions as stacked regime distributions. The figure shows that Pred + EO could not mitigate hidden procedural bias, but the contrast is seen in Pred + CEC, which  the dominant Regime B mass is shifted into Regime A. The full model retained nearly all of CEC's regime improvements while the EO component provided a small increase in Regime C and D samples.


\section{Discussion}
Our results provide clear answers to each research question and also expose an important blind spot in the current algorithmic fairness paradigm. \textbf{RQ1} asked whether outcome-fair models can harbor hidden procedural bias. The answer was shown to be yes across all four datasets; every baseline method that achieved competitive equalized odds scores simultaneously exhibited CEC scores comparable to or worse than the unconstrained model. The regime distribution analysis makes this explicit on German Credit, where DIR, Hardt, and Agarwal place nearly 90\% of individuals in Regime B, meaning that for the vast majority of applicants, the model produced an outcome fair prediction but arrived at it through group-dependent reasoning. The ablation study sharpens this finding further, as adding EO regularization alone did not reduce Regime B at all. In fact, it made unfairness more severe for those individuals. This demonstrates that outcome fairness and procedural fairness are orthogonal and optimizing for one does not address other.

\textbf{RQ2} and \textbf{RQ3} are best answered jointly. The CEC model was the only method that consistently achieved the lowest CEC scores and Regime B fractions across all datasets, while simultaneously delivering competitive equalized odds and statistical parity performance. The ablation shows that regularization alone drove the EO gap to zero as a side effect, suggesting that enforcing consistent reasoning across groups naturally promotes outcome fairness equity but the reverse does not hold. Existing fairness interventions, whether they operate through data transformation (DIR), threshold adjustment (Hardt), constrained optimization (Agarwal, Lagrangian), or adversarial training, all failed to reduce procedural bias because they optimize over predictions without attending to the model's internal explanation structure. The sensitivity analysis confirms that these gains are robust, as CEC and Regime B improve sharply at moderate $\lambda_{\text{CEC}}$ values while F1 remains above 0.82, and the favorable operating region is broad rather than a narrow sweet spot.

\textbf{RQ4} asked whether CEC generalizes across data settings of varying complexity and realism. The progression from synthetic data (controlled bias, known ground truth) through benchmark datasets (German Credit from 1990s Germany, Adult Income from U.S. census data) to real-world mortgage lending (HMDA under active ECOA oversight) demonstrates that CEC model effectiveness is not an artifact of any single domain. Notably, the magnitude of improvement varies in ways that are consistent with the structural properties of each dataset: German Credit and Adult Income, where historical patterns of gender and racial segregation are deeply embedded in feature distributions, showed the largest CEC reductions, while HMDA, where modern regulatory pressure has already reduced the most overt forms of procedural bias, showed a smaller but still meaningful improvement. These findings suggest that procedural bias is a pervasive feature of real lending systems and it persists even under regulatory oversight, and that only becomes visible when the right diagnostic tools are applied.

Our findings carry direct consequences for fair lending compliance. Under ECOA and Regulation B, lenders are prohibited from applying different underwriting standards to different demographic groups, and this prohibition maps precisely onto Regime B in our taxonomy. Current model risk management practices, including those outlined in the US Federal Reserve's SR 11-7 guidance, evaluate models primarily through outcome-based disparate impact testing. Our results demonstrate that this approach might be insufficient because a model can pass every standard disparate impact test while systematically applying different reasoning to applicants from different groups. This gap is particularly concerning as financial institutions increasingly adopt complex neural network models that are harder to audit through traditional means. CEC offers a practical tool to close this gap. At the model development stage, CEC regularization can be incorporated into training pipelines to prevent procedural bias from arising in the first place. At the audit stage, the CEC score and regime distribution provide quantitative evidence of whether a model's reasoning is consistent across groups. The per-individual nature of CEC is especially relevant because, rather than reporting aggregate group statistics, it can identify specific applicants who were subjected to inconsistent reasoning, enabling targeted remediation. As regulators worldwide move toward requiring explainability in automated credit decisions, frameworks that audit, not just what models decide but how they reason, will become increasingly essential.

\section{Conclusion}

We introduced Counterfactual Explanation Consistency (CEC), a framework for detecting and mitigating hidden procedural bias in automated lending models. Our work is motivated by a fundamental gap in the algorithmic fairness literature, where existing methods ensure that models produce equitable outcomes across demographic groups but do not examine whether models arrive at those outcomes through consistent reasoning. We formalized this gap through a four-regime taxonomy and showed, both theoretically and empirically, that outcome fairness and procedural fairness are orthogonal and optimizing for one can actively undermine the other.

CEC addresses this gap by proposing nearest-neighbor counterfactual generation that produces realistic matched pairs without causal assumptions, a consistent baseline principle that isolates discriminatory reasoning from reference-point artifacts in integrated gradient comparisons, and a differentiable training loss that encourages explanation consistency alongside accuracy and outcome fairness. Our experiments across synthetic data, German Credit, Adult Income, and HMDA mortgage data demonstrated that all six baseline fairness methods leave the majority of individuals in Regime B (same prediction, different reasoning), while CEC shifts the dominant regime to Regime A (fully fair) with minimal utility cost. The ablation study reveals that CEC regularization alone drove the equalized odds gap to near zero as a side effect, while the reverse did not hold, suggesting that procedural fairness may be the more fundamental objective from which outcome fairness follows naturally.

\section{Limitations and Future Work}

\textbf{Limitations.} Our framework has several limitations that warrant acknowledgment. First, CEC relies on integrated gradients as the attribution method, which assumes that the straight-line path from baseline to input traverses a meaningful region of the model's decision space. Alternative attribution methods, such as SHAP or attention-based explanations, may yield different consistency assessments. However, we leave the comparative study of attribution methods within the CEC framework to future work. Second, the financial feature set $\mathcal{F}$ must be specified by domain experts, and different choices of $\mathcal{F}$ may lead to different counterfactual matches and CEC scores. While this is inherent to any fairness method that distinguishes legitimate from illegitimate features, it introduces a degree of subjectivity that practitioners must navigate. Third, our experiments use binary protected attributes; extending CEC to multi-valued or intersectional attributes would require addressing combinatorial growth in the number of counterfactual comparisons and label-group baselines. Fourth, the computational overhead of integrated gradients (approximately $3\times$ to $5\times$ standard training time) may limit applicability to very large-scale models, though this cost is incurred only during training and not at inference time.

\textbf{Future Work.} Several directions merit further investigation. CEC scores can be decomposed at the feature level to identify which specific features exhibit the largest attribution disparities across groups, enabling targeted model debugging and more interpretable fairness audits. The relationship between CEC and causal notions of fairness deserves formal analysis: while our nearest-neighbor matching avoids causal assumptions, understanding when and whether CEC scores align with causal path-specific effects would strengthen the theoretical foundations. Extending the framework to continuous or multi-valued protected attributes, regression settings, and non-tabular domains such as text or image-based lending decisions would also broaden its applicability.

\bibliography{aaai25}

\end{document}